\documentclass[12pt]{article}

\usepackage{amsmath, amssymb, amstext, amsthm,epsfig}

\usepackage[usenames]{color}
\usepackage{colortbl}

\newcommand{\ver}[1]{{\color{black}#1}}
\newcommand{\mil}[1]{{\color{black}#1}}

\newcommand{\poly}{\text{poly}}

\newcommand{\cnd}{\mskip 1mu|\mskip 1mu}
\newcommand{\pair}[1]{\langle #1\rangle}

\renewcommand{\phi}{\varphi}
\renewcommand{\epsilon}{\varepsilon}
\renewcommand{\ge}{\geqslant}
\renewcommand{\le}{\leqslant}

\newtheorem{theorem}{Theorem}
\newtheorem{lemma}[theorem]{Lemma}
\newtheorem{corollary}[theorem]{Corollary}
\newtheorem{example}[theorem]{Example}

\theoremstyle{remark}
\newtheorem{remark}{Remark}
\newtheorem{definition}{Definition}

\DeclareMathOperator{\m}{{\mathbf{m}\,}}

\begin{document}

\pagestyle{plain}


\title{Algorithmic statistics, prediction and machine learning}
\author{Alexey Milovanov\\Moscow State University\\
{\tt almas239@gmail.com}}

\maketitle

\begin{abstract}
Algorithmic statistics considers the following problem: given a binary string $x$ (e.g., some experimental data), find a ``good'' explanation of this data. It uses algorithmic information theory to define formally what is a good explanation. In this paper we extend this framework in two directions.

First, the explanations are not only interesting in themselves but also used for prediction: we want to know what kind of data we may reasonably expect in similar situations (repeating the same experiment). We show that some kind of hierarchy can be constructed both in terms of algorithmic statistics and using the notion of a priori probability, and these two approaches turn out to be equivalent (Theorem~\ref{mt}).

Second, a more realistic approach that goes back to machine learning theory, assumes that we have not a single data string $x$ but some set of ``positive examples'' $x_1,\ldots,x_l$ that all belong to some unknown set $A$, a property that we want to learn. We want this set $A$ to contain all positive examples and to be as small and simple as possible. We show how algorithmic statistic can be extended to cover this situation (Theorem~\ref{mlss}).

\end{abstract}


\section{Introduction and notation}

Let $x$ be a binary string, and let $A$ be a finite set of binary strings containing~$x$. Considering $A$ as an ``explanation'' (statistical model) for $x$, we want $A$ to be as simple and small as possible (the smaller $A$ is, the more specific the explanation is). This approach can be made formal in the framework of algorithmic information theory, where the notion of algorithmic (Kolmogorov) complexity of a finite object (a string or a set encoded as a binary string in a natural way) is defined.

The definition and basic properties of Kolmogorov
complexity can be found in the textbooks 
 \cite{LiVit}, \cite{SUV}, for a short survey see \cite{Shen}. 
Informally Kolmogorov complexity of a string $x$ is defined as the minimal length of a program that produces $x$. 
This definition depends on the programming language, but there are optimal languages that make the complexity minimal up to a constant; we fix one of them and denote the complexity of $x$ by $C(x)$.

We also use another basic notion of the algorithmic information theory, the \emph{discrete a priory probability}. Consider a probabilistic machine $A$ without input that outputs some binary string and stops. It defines a probability distribution on binary strings: $m_A(x)$ is the probability to get $x$ as the output of $A$. (The sum of $m_A(x)$ over all $x$ can be less than $1$ since the machine can also hang.) The functions $m_A$ can be also characterized as lower semicomputable semimeasures (non-negative real-valued functions $m(\cdot)$ on binary strings such that the set of pairs $(r,x)$ where $r$ is a rational number, $x$ is a binary string and $r<m(x)$, is computably enumerable, and $\sum_x m(x)\le 1$). There exists a universal machine $U$ such that $m_U$ is maximal (up to $O(1)$-factor) among all $m_A$. We fix some $U$ with this property and call $m_U(x)$ the \emph{discrete a priori probability of $x$}, denoted as $\m(x)$. The function $\m$ is closely related to Kolmogorov complexity. Namely, the value $-\log_2 \m(x)$ is equal to $C(x)$ with $O(\log C(x))$-precision.

Now we can define two parameters that measure the quality of a finite set $A$ as a model for its element $x$: the complexity $C(A)$ of $A$ and the binary logarithm $\log|A|$ of its size. The first parameter measures how simple is our explanation; the second one measures how specific it is. We use binary logarithms to get both parameters in the same scale: to specify an element of a set of size $N$ we need $\log N$ bits of information.

There is a trade-off between two parameters. The singleton $A=\{x\}$ is a very specific description, but its complexity may be high. On the other hand, for a $n$-bit string $x$ the set $A=\mathbb{B}^n$ of all $n$-bit strings is simple, but it is large. To analyze this trade-off, following \cite{Kolmogorov,Koppel}, let us note that every set $A$ containing $x$ leads to a \emph{two-part description of $x$}: first we specify $A$ using $C(A)$ bits, and then we specify $x$ by its ordinal number in $A$, using $\log|A|$ bits. In total we need $C(A)+\log|A|$ bits to specify $x$ (plus logarithmic number of bits to separate two parts of the description). This gives the inequality
$$
C(x)\le C(A)+\log |A| + O(\log C(A))
$$
(the length of the optimal description, $C(x)$, does not exceed the length of any two-part description). The difference 
$$
\delta(x,A)=C(A)+ \log |A|-C(x)
$$
is called \emph{optimality deficiency} of $A$ (as a model for $x$). As usual in algorithmic statistic, all our statements are made with logarithmic precision (with error tolerance $O(\log n)$ for $n$-bit strings), so we ignore the logarithmic terms and say that $\delta(x,A)$ is positive and measures the overhead caused by using two-part description based on $A$ instead of the optimal description for $x$. 

Note that this overhead $\delta(x,A)$ is zero for $A=\{x\}$, so the question is whether we can obtain $A$ that is simpler than $x$ but maintains $\delta(x,A)$ reasonably small. This trade-off is reflected by a curve called sometimes that the \emph{profile} of $x$; this profile can be defined also in terms of randomness deficiency (the notion of $(\alpha,\beta)$-stochasticity introduced by Kolmogorov, see \cite{SUV}, \cite{VerShen}\ver{)},  and in terms of time-bounded Kolmogorov complexity (the notion of depth, see \cite{VerShen}).

In our paper we apply these notions to an analysis of the prediction and learning. In Section~\ref{sec:prediction} we consider, for a given string $x$, all ``good'' explanations and consider their union. Elements of this union are strings that can be reasonably expected when the experiment that produced $x$ is repeated. We show that this union has another equivalent definition in terms of a priori probability (Theorem~\ref{mt}).

In Subsection~\ref{sl} we consider a situation where we start with several data strings $x_1,\ldots,x_l$ obtained in several independent experiments of the same type. We show that all the basic notions of algorithmic statistics can be extended (with appropriate changes) to this framework, as well as Theorem \ref{mt}.

\section{Prediction Hierarchy}\label{sec:prediction}

\subsection{Algorithmic prediction}

Assume that we have some experimental data represented as a binary string~$x$. We look for a good statistical model for $x$ and find some set $A$ that has small optimality deficiency $\delta(x,A)$. If we believe in this model, we expect only elements from $A$ as outcomes when the same experiment is repeated. The problem, however, is that many different models with small optimality deficiency may exist for a given $x$, and they may contain different elements. If we want to cover all the possibilities, we need to consider the union of all these sets, so we get the following definition. In the following definition we assume that $x$ is a binary string of length $n$, and all the sets $A$ also contain only strings of length $n$.

\begin{definition}
Let $x\in\mathbb{B}^n$ be a binary string and let $d$ be some integer. The union of all finite sets of strings $A\subset \mathbb{B}^n$ such that $x\in A$ and $\delta(x,A)\le d$ is called \emph{algorithmic prediction $d$-neighborhood of $x$}. 
\end{definition}

Obviously $d$-neighborhood increases as $d$ increases. It becomes trivial (contains all $n$-bit strings) when $d=n$ (then $\mathbb{B}^n$ is one of the sets $A$ in the union).

\begin{example}
If $x = 0 \dots 0$ (the strings consisting of $n$ zeros), then $x'$ belongs to 
$d$-neighborhood of $x$ iff $C(x') \lesssim d$
\end{example}

\begin{example}
If $x$ is a random string of length $n$ (i. e. $C(x) \approx n$) then \ver{the 
$d$-neighborhood of $x$ contains all strings of length $n$
provided $d$ is greater than some function of order $O(\log n)$}.
\end{example}

\subsection{Probabilistic prediction}
\label{probpred}
There is another natural approach to prediction. Since we treat the experiment as a black box (the only thing we know is its outcome $x$), we assume that the possible models $A\subset\mathbb{B}^n$ are distributed according to their a priori probabilities, and consider the following two-stage process. First, a finite set is selected randomly: a non-empty set $A$ is chosen with probability $\m(A)$ (note that a priori probability can be naturally defined for finite sets via some computable encoding). Second, a random element $x$ of $A$ is chosen uniformly. In this process every string $x$ is chosen with probability 
$$
\sum_{A\ni x} \m(A)/|A|,
$$
and it is easy to see that this probability is equal to $\m(x)$ up to a
$O(1)$-factor. Indeed, the formula above defines a lower semicomputable function of $x$, so it does not exceed $\m(x)$ more than by $O(1)$-factor. On the other hand, if we restrict the sum to the singleton $\{x\}$, we already get $\m(x)$ up to a constant factor. So this process gives nothing new in terms of the final output distribution on the outcomes $x$. Still the advantage is that we may consider, for a given pair of strings $x$ and $y$, the conditional probability
$$
p(y \cnd x )=\Pr [y\in A \mid \text{the output of the two-stage process is $x$}].
$$

In other words, by definition
\begin{equation}\label{eq:definition-of-p}
p(y   \cnd x )=\frac{\sum_{A\ni x,y} \m(A)/|A|}{\sum_{A\ni x} \m(A)/|A|}.
\end{equation}
As we have said, the denominator equals $\m(x)$ up to $O(1)$-factor, so 
\begin{equation}
\label{defpyx}
p(y  \cnd x )=\frac{\sum_{A\ni x,y} \m(A)/|A|}{\m(x)}
\end{equation}
up to $O(1)$-factor. Having some string $x$ and some threshold $d$, we now can consider all strings $y$ such that $p(y\cnd x)\ge 2^{-d}$ (we use the logarithmic scale to facilitate the comparison with algorithmic prediction). These strings could be considered as plausible \ver{ones} to appear when repeating the experiment of unknown nature that once gave $x$.  

Our main result shows that this approach is essentially equivalent to the algorithmic prediction. By a technical reason we have to change slightly the random process that defines $p(y\cnd x)$. Namely, it is strange to consider models that are much more complex than $x$ itself, so we consider only sets $A$ whose complexity does not exceed $\poly(n)$; any sufficiently large polynomial can be used here (in fact, $4n$ is enough). So we assume that the sums in 
(\ref{eq:definition-of-p}) and \ver{\eqref{defpyx}, and in} similar formulas in the sequel \ver{are} always restricted to sets $A\subset \mathbb{B}^n$ that have complexity at most $4n$, and take this modified version of (\ref{eq:definition-of-p}) as a final definition for $p(y\cnd x)$.

\begin{definition}
Let $x$ be a binary string and let $d$ be an integer. The set of all strings $y$ such that $p(y\cnd x)\ge 2^{-d}$ is called \emph{probabilistic prediction $d$-neighborhood of $x$}.
\end{definition}

We are ready to state the main result of this section.
 
\begin{theorem}\label{mt}

\textup{(a)} For every $n$-bit string $x$ and for every $d$ the algorithmic prediction $d$-neighborhood is contained in probabilistic prediction $d + O(\log n)$-neighborhood.

\textup{(b)} For every $n$-bit string $x$ and for every $d$ the probabilistic prediction $d$-neighborhood of $x$ is contained in algorithmic prediction $d + O(\log n)$-neighborhood.
\end{theorem}
The next section contains the proof of this result; later we show some its possible extensions.

\subsection{The proof of the Theorem~\ref{mt}}
\begin{proof}[Proof of \textup{(a)}] 
This direction is simple. Assume that some string $y$ belongs to the algorithmic prediction $d$-neighborhood of $x$, i.e.,  there is a set $A$ containing $x$ and $y$ such that $C(A) + \log|A| \le C(x) + d$. We may assume without loss of generality that $d\le 2n$ otherwise all $n$-bit string belong to probabilistic prediction $d$-neighborhood of $x$ (take $A=\mathbb{B}^n$). Then the inequality for $C(A)+\log |A|$  implies that complexity of $A$ does not exceed $4n$, so 
the set $A$ is included in the sum. This inequality implies also that 
$$
\frac{\m(A) / |A|}{\m(x)} \ge 2^{-d - O(\log n)}
$$
(as we have said, $-\log\m(u)$ equals $C(u)+O(\log C(u))$). This fraction is one of terms in the sum that defines $p(y\cnd x)$, so $y$ belongs to the probabilistic prediction $d + O(\log n)$-neighborhood of $x$. 
\end{proof}

Before proving the second part (b), we need to prove \ver{a} 
technical lemma. It is inspired by  \cite[Lemma 6]{VerVit} where \ver{it was shown}
that if a string belongs to many sets of bounded complexity, 
then one of them has even smaller complexity. 
We generalize that result as follows.

\ver{
\begin{lemma}\label{ml}
Assume that sets $L$ and $R$ consist 
of finite objects (in particular, 
Kolmogorov complexity $C(v)$ is defined for $v\in L$).
Assume that $R$ is has at most $2^n$ elements.
Let $G$ be a finite bipartite graph 
where $L$ and $R$ are the sets of its left and right nodes, respectively. 
Assume that a right 
node $x$ has at least $2^k$ neighbors of Kolmogorov
complexity at most $i$.
Then $x$ has a neighbor of complexity at most
$i - k + O(C(G)+\log (k + i+n))$.
Here $C(G)$ stands for the length 
of the shortest program that 
given any $v\in L$ outputs a list of its neighbors.
\end{lemma}

\begin{proof}
Let us enumerate left nodes that have complexity at most $i$. 
We start a selection process: some of them are marked (=selected) 
immediately after they appear in the enumeration. 
This selection should satisfy the following requirements:
\begin{itemize}
\item at any moment every right node that has at least $2^k$ neighbors
among enumerated nodes, has a marked neighbor;
\item the total number of marked nodes does not exceed $2^{i-k} p(i, k, n)$
for some polynomial $p$ (fixed in advance). 
\end{itemize}
If we have such a selection strategy of  
complexity $C(G)+O(\log(i+k+n))$, this implies that the right node $x$ 
has a neighbor of 
complexity at most 
$$
i - k + O(C(G)+\log ( k + i+n)),
$$ 
namely, any its marked neighbor (that marked neighbor
can be specified by its number in the list of all marked nodes).

To prove the existence of such a strategy, 
let us consider the following game. 
The game is played by two players, who alternate moves. 
The maximal number of moves is $2^i$. 
At each move the first player plays a left node, 
and the second player replies saying whether she marks that node
or not. The second player loses if  the number of 
marked nodes exceeds $2^{i-k+1} ( n + 1) \ln 2$ or if after some 
of her moves there exists 
a right node $y$ that has at least $2^k$ neighbors 
among the nodes chosen by the first player but has no 
marked neighbor. (The choice of the bound $2^{i-k+1} ( n + 1) \ln 2$
will be clear from the probabilistic estimate below.) Otherwise she wins.

Assume first that the set $L$
of left nodes is finite (recall that the 
set of right nodes is finite by assumption).
Then our game is a finite game with full information, an hence
one of the players has a winning
strategy. We claim that the second player can win. 
If it is not the case, the first
player has a winning strategy. We get a contradiction by showing that the 
second
player has a probabilistic strategy that wins with positive probability 
against any
strategy of the first player. So we assume that some strategy of the
first player is fixed, and consider the following simple probabilistic
strategy of the second player: 
every node presented by the first player is marked
with probability $p = 2^{-k} ( n + 1) \ln 2$.
The expected number of marked nodes is $p 2^i = 2^{i-k} ( n + 1) \ln 2$.
By Markov's inequality, the number of marked nodes exceeds the expectation by a factor of  $2$ with
probability less than $\frac{1}{2}$. So it is enough to show that the
second bad case (after some move there exists a right 
node  $y$ that has $2^k$ neighbors 
among the nodes chosen by first player but has no marked neighbor) 
happens with probability at most $\frac{1}{2}$.

For that, it is enough to show that for every node right node $y$
the probability of this bad
event is less than $\frac{1}{2}$ divided by the number $|R|$ of 
right nodes.
Let us estimate this probability. If $y$ has $2^k$ (or more)
neighbors, the second player had (at least) $2^k$ chances to mark its
neighbor (when these $2^k$ nodes were presented by the first player), 
and the
probability to miss all $2^k$ these chances is at most $(1 - p)^{2^k}$. 
The choice of $p$
guarantees that this probability
is less than $2^{- n-1}\le (1/2)/|R|$. Indeed, using the bound 
$1 -x \le e^{-x}$, it is easy to
show that
    $$
(1 - p)^{2^k} \le e^{\ln 2 \cdot(- n -1)} = 2^{- n-1}.
    $$

We have proven that the winning strategy exists but have not yet estimated
is complexity. A winning strategy can be found be an exhaustive search
among all the strategies. The set of all strategies is finite
and the game is specified by $G$, $i$ and $k$. Therefore the complexity 
of the first found winning strategy is at most $C(G)+O(\log(i+k))$.

Thus the Lemma~\ref{ml} is proven in the case when $L$ is a finite set.
To extend the proof to general case, notice that the winning
condition depends only on the neighborhood 
of each left node. The worst graph for the the second player 
is the following ``model'' graph.
It has $2^{2^n+i}$ left nodes and $2^n$ right nodes and
each of $2^{2^{n}}$ possible 
neighborhoods is shared by $2^i$ left nodes. A winning
strategy for such a graph can be found from $n$, $i$ and $k$
and hence its complexity is logarithmic in $n+i+k$. That 
strategy can be translated to the game associated with
the initial graph, this translation increases the complexity by $C(G)$,
as we have to translate each left node played by the first
player to a left node of the model graph.
\end{proof}
}

Having in mind future applications in Subsection~\ref{sort}, 
we will consider in the next statement
an arbitrary decidable family $\mathcal{A}$ 
of finite sets though in this section we need only the case when 
$\mathcal{A}$ contains all finite sets.

\begin{corollary}
\label{cfl}
\ver{Let $\mathcal{A}$ be a decidable family of finite sets. 
Assume that  $x_1,\ldots,x_l$ are strings of length $n$. 
Denote by $\mathcal{A}_m^n$ all subsets of $\mathbb{B}^n$ of complexity at most $m$. Then the sum 
$$ S := \sum_{A \in \mathcal{A}_m^n,\ x_1,\dots,x_l\in A} \frac{\m(A)}{|A|} $$
equals to its maximal term up to a factor of 
$2^{O(\log(n + m + l))}$.}
\end{corollary}

\begin{proof}[Proof of the corollary]
%
Let $M$ denote the maximal term in the sum $S$.
Obviously the sum $S$ is equal to the sum over $i\le m$ and $j\le n$  
of sums 
\begin{equation}
\label{eq1}
\sum_{\substack{A \in \mathcal{A}_m^n \\ C(A) = i\\ \log|A| = j\\ x_1,\dots,x_l\in A}} \frac{\m(A)}{|A|}.
\end{equation}
As there are $(m+1)(n+1)$ such sums, we only need to prove that 
each sum~\eqref{eq1} is at most $M \cdot 2^{O(\log n + m + l)}$. 
In other words, we have to show that for all $i,j$ 
there is a set $H\in \mathcal{A}_m^n$
with $x_1,\dots,x_l\in A$ such that 
$\frac{\m(H)}{|H|}$ is greater than
the sum~\eqref{eq1} up to a factor of $2^{O(\log(n + m + l))}$.

To this end fix $i$ and $j$. Since $\m(u) = 2^{-C(u) - O(\log C(u))}$,
the sum~\eqref{eq1} equals
\begin{equation}
\label{sum_ij}
\sum_{\substack{A \in \mathcal{A}_m^n \\ C(A) = i\\ \log|A| = j\\ x_1,\dots,x_l\in A}} 2^{-C(A) - \log|A|+O(\log( n + m))} =
\sum_{\substack{A \in \mathcal{A}_m^n \\ C(A) = i\\ \log|A| = j\\ x_1,\dots,x_l\in A}} 2^{-i - j+
O(\log( n + m))}
\end{equation}
All the terms in the sum~\eqref{sum_ij} coincide and thus 
the sum~\eqref{sum_ij} is equal to $2^{-i - j+
O(\log( n + m))}$ times the number 
of sets $A \in \mathcal{A}_m^n$ with $C(A) = i$, $\log|A| = j$,
$x_1,\dots,x_l\in A$. Let $k$ denote the floor of the binary logarithm
of that number. 

Consider the bipartite graph whose left nodes are finite
subsets from $\mathcal A^n$ of cardinality at most $2^j$, 
right nodes are $l$-tuples
of $n$-bit strings and a left node $A$ is adjacent 
to a right node  $\pair{x_1,\dots,x_l}$ if all $x_1,\dots,x_l$
are in $A$. The complexity of this graph is $O(\log(n+l+j))$
and the logarithm of the number of right nodes is $nl$. 
By Lemma \ref{ml} 
there is a set $H \in \mathcal{A}_m^n$ of log-size $j$ and 
complexity at most $i-k + O(\log( i + j + k + n + l)) = i - k + O(\log(l + m + n))$ with $x_1,\dots,x_l\in A$. The fraction 
$\frac{\m(H)}{|H|}$ is equal to $2^{-(i-k)-j}$ 
up to a factor of $2^{O(\log(n + m + l))}$.

Recall that the sum~\eqref{sum_ij} 

equals to $2^k2^{-i-j}$ up to the same factor and thus we are done.
\end{proof}

\mil{
\begin{remark}

Consider the following case of Corollary \ref{cfl}: $\mathcal{A}$ is the family off all finite subsets, $l=1$. As was shown in Subsection \ref{probpred} the sum $
\sum_{A\ni x} \m(A)/|A|,
$ is equal to $\m(x)$ up to a \emph{constant} factor.
 
By this reason,  we expect that the accuracy in the corollary can be improved.
 
\end{remark}
}
\begin{proof}[Proof of \textup{(b)}] 
Let $y$ be some string that belongs to probability prediction $d$-neighborhood for~$x$.
According to (\ref{defpyx}),  it implies that
$$
\sum_{A \ni x, y} \frac{m(A)}{|A|} \ge \m(x)2^{-d-O(\log n)} = 2^{- d - C(x) - O(\log n)}
$$

Now we will use Corollary \ref{cfl} for $l=2$, $x_1=x$, $x_2 = y$, $m = 4n$ and the family \ver{of} all sets as $\mathcal{A}$. By this corollary there is a set $A \ni x, y$ such that $\m(A)/|A| = 2^{- d - C(x) - O(\log n)}$, so: $C(A) + \log|A| - C(x) \le d + O(\log n)$, i. e. $y$ belongs to the algorithmic prediction $d + O(\log n)$-neighborhood of $x$.
\end{proof}
 
\subsection{Sets of restricted type}
\label{sort}
In some cases we know \emph{a priori} what sets could be possible explanations, and are interested only in models from this class. To take this into account, we consider some family $\mathcal{A}$ of finite sets, and look for sets $A$ in $\mathcal{A}$ that contain the data string $x$ and are ``good models'' for $x$. This approach was used in ~\cite{VerVit}; it turns out that many results of algorithmic statistics can be extended to this case (though sometimes we get  weaker versions with more complicated proofs).
 
In this section we show that Theorem \ref{mt} also has an analog for arbitrary decidable family $\mathcal{A}$.  
The family of all subsets of $\mathbb{B}^n$ that belong to $\mathcal{A}$ is denoted by $\mathcal{A}^n$. 

\ver{First we consider the case when for each string $x$ the set
$\mathcal{A}$  contains the singleton $\{x\}$}. 

Let us  define probability prediction neighborhood for a $n$-bit string $x$ with respect to $\mathcal{A}$.
Again we consider a two-stage process: first,  some set of $n$-bit strings from $\mathcal{A}$ is chosen with probability $\m(A)$. Second, a random element in $A$ is chosen uniformly. Again, we have to assume that we choose sets whose complexity is not greater than $4n$. A value $p_{\mathcal{A}}(y|x)$ is then defined as the conditional probability of $y\in A$ with the condition ``the output of the two-stage process is $x$'':
\begin{equation}
\label{src}
     p_{\mathcal{A}}(y \cnd x)= \frac{\sum_{A\ni x,y} \m(A)/|A|}{\sum_{A\ni x} \m(A)/|A|}
\end{equation}
\ver{Here} the sum is taken over all sets  in $\mathcal{A}^n$ that  have complexity at most $4n$. 

Again as in Subsection \ref{probpred} the denominator equals $\m(x)$  up to $O(1)$-factor (because $\{x\} \in \mathcal{A}$\ver{)}, so:
\begin{equation}
\label{src0}
     p_{\mathcal{A}}(y \cnd x)= \frac{\sum_{A\ni x,y} \m(A)/|A|}{\m(x)}
\end{equation}
up to $O(1)$-factor.

Then \emph{$\mathcal{A}$-probabilistic prediction $d$-neighborhood} is defined naturally: a string $y$ belongs to this neighborhood if $p_{\mathcal{A}}(y|x) \ge 2^{-d}$. \ver{The}
\emph{$\mathcal{A}$-algorithmic prediction $d$-neighborhood} for $x$ is defined as follows: a string $y$ belong to it if there is a set $A \ni x, y$ that belongs to $\mathcal{A}^n$ such that $\delta(x, A) \le d$.

Now we are ready to state an analog of Theorem \ref{mt}:

\begin{theorem} \label{rc0}
Let $\mathcal{A}$ be a decidable family of binary strings containing all singletons. Then:

\textup{(a)} For every $n$-bit string $x$ and for every $d$ the $\mathcal{A}$-algorithmic prediction $d$-neighborhood is contained in $\mathcal{A}$-probabilistic prediction $d + O(\log n)$-neighborhood.

\textup{(b)} For every $n$-bit string $x$ and for every $d$ the $\mathcal{A}$-probabilistic prediction $d$-neighborhood of $x$ is contained in $\mathcal{A}$-algorithmic prediction $d + O(\log n)$-neighborhood.
\end{theorem}

\begin{proof}[Proof of \textup{(a)}]
The proof is similar to the proof of Theorem \ref{mt} (a). Assume that a string $y$ belongs to the algorithmic prediction $d$-neighborhood for $x$, i.e.,  there is a set $A \in \mathcal{A}^n$ containing $x$ and $y$ such that $C(A) + \log|A| \le C(x) + d$. \ver{If  $d >3n$, then the statement is trivial.
Indeed, there is a set $A' \in \mathcal{A}^n$ 
that contains $x$ and $y$ such that $\delta(x, A') \le 3n$. 
To prove this, we can not set  $A'=\mathbb{B}^n$ 
any more, as this set may not belong to $\mathcal{A}$. 
However we may let $A'$ be the first set in 
$\mathcal{A}^n$, that contains $x$ and $y$.
The complexity of this set is not greater than $|x| + |y| \le 2n$ and log-size is not greater than $n$. Thus $\delta(x,A')\le 3n$}. 
The rest of the proof is completely similar to the proof of Theorem \ref{mt} (a).
\end{proof}

\begin{proof}[Proof of \textup{(b)}]
The proof is similar to the proof of Theorem \ref{mt} (b).
\end{proof}

\ver{Now we state and prove Theorem \ref{rc0} in general case 
(for families $\mathcal{A}$ that may not contain all singletons).
In the case $x \in \bigcup\mathcal{A}^n$}, where $n=|x|$, the 
definition of $\mathcal{A}$-probability prediction neighborhood 
remains the same. Otherwise, if $x \notin \bigcup\mathcal{A}^n$, the string
$x$ can not appear in the two-stage process, so in this case we define 
$\mathcal{A}$-probability prediction $d$-neighborhood for $x$ as the empty set for every $d$.
Notice, that now we can not rewrite (\ref{src}) as (\ref{src0}) because $\{x\}$ may not belongs to $\mathcal{A}$. 

Now we define $\mathcal{A}$-algorithmic prediction neighborhood. There is a subtle point that should be taken into account: it may happen that there is no set $A \in \mathcal{A}$ containing $x$ such that $\delta(x, A) \approx 0$. By this reason we include in the algorithmic prediction neighborhood of $x$ the union of all sets $A$ in $\mathcal{A}$, such that $\delta(x, A)$ is as small as it \ver{is}
possible: 

\begin{definition}
\label{aprc}
Let $x\in\mathbb{B}^n$ be a binary string, let $d$ be some integer and let $\mathcal{A}$ be some family of sets. The union of all finite sets in $\mathcal{A}^n$ such that $x\in A$ and every $B \in \mathcal{A}^n$ that contains $x$ satisfies the inequality: $\delta(x,A)\le \delta(x, B) + d$ is called \emph{$\mathcal{A}$-algorithmic prediction $d$-neighborhood of $x$}. 
(\ver{In} other words,  $d$-neighborhood \ver{includes} all sets $A$ \ver{whose
$\delta(x,A)$ is at most $d$ more than the minimum.})
\end{definition}

\begin{theorem}
\label{rc}
Let $\mathcal{A}$ be a decidable family of binary strings.
Then:

\textup{(a)} For every $n$-bit string $x$ and for every $d$ the $\mathcal{A}$-algorithmic prediction $d$-neighborhood is contained in $\mathcal{A}$-probabilistic prediction $d + O(\log n)$-neighborhood.

\textup{(b)} For every $n$-bit string $x$ and for every $d$ the $\mathcal{A}$-probabilistic prediction $d$-neighborhood of $x$ is contained in $\mathcal{A}$-algorithmic prediction $d + O(\log n)$-neighborhood.
\end{theorem}

\ver{Notice that if $x\notin \bigcup\mathcal A^n$ then
both algorithmic and prediction neighborhoods are empty and the statement
is trivial. Therefore in the proof we will assume that this is not the case.}

\begin{proof}[Proof of \textup{(a)}]
The proof is completely similar to the proof of Theorem \ref{rc0}.
\end{proof}

\begin{proof}[Proof of \textup{(b)}]
Let $y$ be some strings that belongs to probability prediction $d$-neighborhood for $x$, \ver{that is,}
\begin{equation}
\label{rcgc}
 \sum_{A \ni x, y}\frac{\m(A)}{|A|} \ge 2^{-d}\sum_{A \ni x}\frac{\m(A)}{|A|} 
 \end{equation}
\ver{Let 
$$
A_x=\arg\max\{\m(A)/|A|\mid x\in A\in\mathcal A^n\}
$$ 
and 
$$
A_{xy}=\arg\max\{\m(A)/|A|\mid x,y\in A\in\mathcal A^n\}.
$$ 
Recall that $\frac{\m(A)}{|A|}=2^{-C(A)-\log|A|}$ (up to a $2^{O(\log n)}$ factor)
and} by Corollary \ref{cfl} the sums in both parts of the equality are equal to \ver{their} largest terms (again up to $2^{O(\log n)}$ factor). 
Therefore,
$$
2^{-C(A_{x,y}) - \log|A_{x,y}|} \ge 2^{-d - O(\log n)}2^{-C(A_{x}) - \log|A_{x}|},
$$
\ver{which means that 
$\delta(x, A_{x,y}) \le \delta(x, A_x) + d + O(\log n)$. 
Hence} $y$ belongs $\mathcal{A}$-algorithmic prediction $d + O(\log n)$-neighborhood of $x$. 
  \end{proof}

\subsection{Prediction for several examples}
\label{sl}
Consider the following situation: we have not one but 
several strings $x_1 ,\ldots, x_l \in \mathbb{B}^n$ that 
are experimental data. \ver{We know that 
they were drawn independently 
with respect to 
the uniform probability distribution in some unknown set $A$.} 
We want to explain these observation data, i. e. to find \ver{an
appropriate set $A$}. Again we \ver{measure the quality of explanations by} 
two parameters: $C(A)$ and $\log|A|$.

\ver{In this section we will extend previous results
to this scenario.}
%
%
Again 
we assume that we know a priori which sets could be possible explanations. 
So, we  consider only sets from a decidable 
family of sets $\mathcal{A}$.

\ver{Let $\overrightarrow{x}$ denote the tuple  $x_1,\ldots ,x_l$}. 
Let $A\subset\mathbb B^n$ 
be a set that contains all strings from $\overrightarrow{x}$. 
Then we can restore $\overrightarrow{x}$ from $A$ and indexes of 
strings from $\overrightarrow{x}$ in $A$ and hence we have :
$$
C(\overrightarrow{x}) \le C(A) + l\log|A| + O(\log n).  
$$
Therefore it is natural to define the 
\emph{optimality deficiency} of $A \ni \overrightarrow{x}$
by the formula
$$
\delta(\overrightarrow{x},A):= C(A) + l\log|A| - C(\overrightarrow{x}).
$$
The definitions of the 
$\mathcal{A}$-algorithmic prediction $d$-neighborhood of the tuple
$\overrightarrow{x}$ 
is obtained from Definition \ref{aprc} by changing $x$ to 
$\overrightarrow{x}$.

In a similar way we modify 
the definition of the $\mathcal{A}$-probabilistic 
prediction neighborhood.
Again we consider a two-stage process: first, a set of $n$-bit strings 
from $\mathcal{A}$ is chosen with probability $\m(A)$. 
Second, $l$ random elements in $A$ are chosen uniformly and independently
on each other. Again, by technical reason, we assume, 
that we consider only sets whose complexity is not greater then $(l+3)n$.
The value $p_{\mathcal{A}}(y|\overrightarrow{x})$ is  
defined as the conditional probability of $y \in A$ 
under the condition [the output of this two-stage process is equal to 
$\overrightarrow{x}$]:
$$
     p_{\mathcal{A}}(y \cnd \overrightarrow{x})= \frac{\sum_{A\ni \overrightarrow{x},y} \m(A)/|A|^l}{\sum_{A\ni \overrightarrow{x}} \m(A)/|A|^l}
$$
Here both sums are taken over all sets $A\in\mathcal{A}^n$ that 
have complexity at most $n(l + 3)$. 
(If no such set contains $x$ then 
$p_{\mathcal{A}}(y \cnd \overrightarrow{x})=0$.)
By definition,
a string $y$ belongs to $\mathcal{A}$-probabilistic prediction $d$-neighborhood for $\overrightarrow{x}$ if 
$p_{\mathcal{A}}(y \cnd \overrightarrow{x}) \ge 2^{-d}$.

Now we are ready to state an analog of Theorem \ref{rc}:
\begin{theorem}
\label{mlss}
Let $\mathcal{A}$ be a decidable family of binary strings. Then:

\textup{(a)} For every $l$ $n$-bit strings $\overrightarrow{x}$ and for every $d$ the $\mathcal{A}$-algorithmic prediction $d$-neighborhood is contained in $\mathcal{A}$-probabilistic prediction $d + O(\log (n+l))$-neighborhood of $\overrightarrow{x}$.

\textup{(b)} For every $l$ $n$-bit strings $\overrightarrow{x}$ and for every $d$ the $\mathcal{A}$-probabilistic prediction $d$-neighborhood of $\overrightarrow{x}$ is contained in $\mathcal{A}$-algorithmic prediction $d + O(\log(n+l))$-neighborhood of $\overrightarrow{x}$.
\end{theorem}

\begin{proof}
The proof is entirely similar to the proof of Theorem \ref{rc}, 
but now Corollary \ref{cfl} is applied  for $l$ and $l+1$ strings so 
the accuracy becomes $O(\log(n+l))$.
\end{proof}
\section{Non-uniform probability distributions}\label{s3}
We have considered so far only uniform probability distributions
as statistical hypotheses. The paper~\cite[Appendix II]{VerVit2002}
justifies such a restriction: it was observed there that 
for every data string $x$ and for probability distribution $P$   
there is a finite set $A\ni x$ that is not worse than
$P$ as an explanation for $x$ (with logarithmic accuracy).
However, if the data consists of more than one string,
then this is not the case. Now, we will explain this in more details. 

The quality of a probability distribution $P$ as an explanation for
the data $x_1,\ldots,x_l$ is measured be the following two parameters:
\begin{itemize}
\item the complexity $C(P)$ of the distribution $P$,
\item $-\log(P(x_1)\dots P(x_l))$ (the smaller this parameter
is the larger is the likelihood to get the tuple $\overrightarrow{x}$
by independently drawing  $l$ strings with respect to $P$).
\end{itemize}

We consider only distributions over finite sets such 
that the probability of every outcome is a rational number.
The complexity of such a distribution is defined as the complexity
of the set of all pairs $\pair{y,P(y)}$ ordered lexicographically.

If $P$ is a uniform distribution over a finite set $A$
then the first parameter becomes $C(A)$ and the second one
becomes $-l\log|A|$. If $l=1$ then for every
pair $x,P$ there is a finite set  $A\ni x$ such that 
both $C(A),\log|A|$ are at most $C(P),-\log P(x)$
with the accuracy $O(\log|x|)$. 
Indeed, let $A=\mathbb B^n$ if $P(x)\ge2^{-n}$ and
$$
A=\{x\in\mathbb B^n\mid P(x)\ge 2^{-i}\}
$$ 
if 
$2^{-i}\le P(x)<2^{-i+1}\le 2^{-n}$. In both cases 
we have $C(A)\le C(P)+O(\log n)$
and $\log|A|\le -\log P(x)+1$.

For $l=2$ this is not the case:
\begin{example}
Let $x_1$ be a random string of length $2n$
and $x_2=00\dots 0y$ be a string of length $2n$
where  $y$ is a random string of length $n$
independent of $x_1$ (that is, $C(x_1,x_2)=3n+O(1)$).
A plausible explanation of such data is the following:
the strings $x_1,x_2$ were drawn independently 
with the respect the distribution $P$ where
half of the probability is uniformly 
distributed over all strings of length $2n$
and the remaining half is  uniformly 
distributed over all strings of length $2n$ starting with $n$
zeros. The complexity of this distribution $P$
is negligible ($O(\log n)$) and the second parameter $-\log(P(x_1)P(x_2))$
is about $3n$. On the other hand there is no simple 
set $A$ containing both strings $x_1,x_2$ with
$2\log|A|$ being close to  $3n$. Indeed, for 
every set $A$ containing $x_1$ we have 
$C(A)+\log|A|\ge 3n-O(\log n)$ and
hence $2\log|A|\ge 6n-2C(A)-O(\log n)\gg 3n$ (the last inequality holds
provided $C(A)$ is small).
\end{example}

Therefore we will not restrict the class of statistical
hypotheses to uniform distributions. We will show
that the main result of ~\cite{VerVit2002} 
(Theorem~\ref{profiles_opt_stoch} below) 
translates to the case of several strings, i.e., to the case $l>1$
(Theorem~\ref{profiles_opt_stoch_many} below).

\subsection{The profile of a tuple $x_1,\dots,x_l$}

Fix $x_1,\dots,x_l\in\mathbb B^n$. As above,
we will denote  by $\overrightarrow{x}$
the tuple  $x_1,\dots,x_l$. The optimality deficiency is defined
by the formula 
$$
\delta(\overrightarrow{x},P)=C(P)-\log(P(x_1)\dots P(x_l))-
C(\overrightarrow{x}).
$$
This value is non-negative up to $O(\log(n+l))$,
since given $P$ and $l$ we can describe 
the tuple $\overrightarrow{x}$ in  
$-\log(P(x_1)\dots P(x_l))+O(1)$ bits, using the Shannon-Fano code.

\begin{definition}
The profile $P_{\overrightarrow{x}}$ of the tuple 
$\overrightarrow{x}$ is defined
as the set of all pairs $\pair{a,b}$ of naturals 
such that there is a probability distribution $P$ of Kolmogorov
complexity at most $a$ with
$\delta(\overrightarrow{x},P)\le b$.
\end{definition}

Loosely speaking,
a tuple of strings $\overrightarrow{x}$ is called \emph{stochastic} 
if there is a simple distribution $P$ such that 
$\delta(\overrightarrow{x}, P) \approx 0$. 
In other words, if $(a,b) \in P_{\overrightarrow{x}}$ for $a, b \approx 0$. 
Otherwise it is called \emph{non-stochastic}.  
In one-dimensional case non-stochastic objects were studied, for example, 
in \cite{Sh}, \cite{VerVit2002}. However, 
in the one-dimensional case we can not present explicitly a 
non-stochastic object. In the two-dimensional case 
the situation is quite different: 
let $x_1$ be a random string of length $n$ and let $x_2=x_1$.
For such pair $x_1,x_2$ there is no simple distribution 
$P$ with small $\delta(\pair{x_1,x_2},P)$. 
Indeed, for any probability distribution $P$ we have
$C(P) - \log P(x_i) \ge C(x_i)=n$ for $i=1,2$ (with accuracy $O(\log n)$).
Adding these inequalities we get
$$
2C(P) - \log (P(x_1)P(x_2)) \ge 2n.
$$
Hence
$\delta(\pair{x_1,x_2},P)\ge 2n-C(P)-C(x_1,x_2)=n-C(P)$,
which is very large provided $C(P)\ll n$.

In general, if strings $x_1$ and $x_2$ have much 
\emph{common information} (i. e. $C(x_1, x_2) \ll C(x_1) + C(x_2)$), 
then the pair $\pair{x_1, x_2}$ is non-stochastic. 
There is also a non-explicit
example of a non-stochastic pair of strings: consider any pair 
whose first term is non-stochastic. 
There is no good explanation for the first term, 
hence there is no good explanation for the whole pair.

The first example suggests
the following question: is the profile of the pair of strings $x_1, x_2$ 
determined by $C(x_1),C(x_2),C(x_1,x_x)$, $P_{x_1}$, $P_{x_2}$ 
and $P_{[x_1,x_2]}$? Here
$[x_1,x_2]$ denotes  the concatenation of strings $x_1$ and $x_2$. 
Notice that $P_{[x_1,x_2]}$ denotes the 1-dimensional profile of the string
$[x_1,x_2]$ and is not to be confused with
$P_{x_1,x_2}$, which is the 2-dimensional profile of the pair of strings
$x_1,x_2$. 
The following theorem is the main result of Section~\ref{s3}.
It provides a negative answer to this question.

\begin{theorem}
\label{profile_common_inf}
For every $n$ there are strings $x_1$, $x_2$, $y_1$ and $y_2$ of length 
$2n$ such that:

1) The sets $P_{x_1}$ and $P_{y_1}$, $P_{x_2}$ and $P_{y_2}$, $P_{[x_1,x_2]}$ and $P_{[y_1, y_2]}$ are at most $O(\log n)$ apart.
 
2) $C(x_1)=C(y_1)+ O(\log n)$,   $C(x_2)=C(y_2)+ O(\log n)$,  
 $C(x_1,x_2)=C(y_1,y_2)+ O(\log n)$.

3) However the distance between $P_{x_1, x_2}$ and $P_{y_1,y_2}$ is greater than 
$0.5n-O(\log n)$. 
(We  say  that  the  distance  between  two  sets
$R$ and $Q$ is  at  most
$\epsilon$ if $R$ is contained in
$\epsilon$-neighborhood, with respect to $L_{\infty}$-norm, of $Q$,
and vice versa.)
\end{theorem}   

The proof of this theorem is presented in Appendix.

\subsection{Randomness deficiency}

In this subsection we introduce multi-dimensional randomness deficiency
and show that the main result of~\cite{VerVit2002} relating  
1-dimensional randomness deficiency and optimality deficiency
translates to any number of strings.

The 1-dimensional randomness deficiency of a string
$x$ in a finite set $A$ was defined by Kolmogorov
as  $d(x|A)=\log|A| - C(x|A)$. It is always non-negative 
(with $O(\log|x|)$ accuracy), as 
we can find $x$ from $A$ and the index of $x$ in $A$.  
For most elements $x$ in any set $A$ the randomness deficiency 
of $x$ in $A$ is negligible.
More specifically, the fraction of $x$ in $A$
with randomness deficiency greater than $\beta$ is less than 
$2^{-\beta}$. The randomness deficiency measures how non-typical
looks $x$ in $A$.   


\begin{definition}
The set of all pairs $(a,b)$ such that there is a set $A \ni x$ of 
complexity at most $a$ and $d(x|A) \le b$ is called the \emph{stochasticity profile} of $x$ and is denoted by $Q_x$
\end{definition}

To distinguish  profiles $P_x$ and $Q_x$ we will call $P_x$ the  \emph{optimality profile} in the sequel. 
Surprisingly, the sets $P_x$ and $Q_x$ almost coincide:
\begin{theorem}[\cite{VerVit2002}]
\label{profiles_opt_stoch}
For every string $x$ of length $n$ the distance between $P_x$ and $Q_x$ is at
most $O(\log n)$.
\end{theorem}   

The multi-dimensional randomness deficiency is defined in the following way.
For a tuple of strings 
$\overrightarrow{x} = x_1, \ldots, x_l$ and a distribution $P$ 
let 
$$
d(\overrightarrow{x}|P) = - \log( P(x_1)\dots P(x_l)) - C(x_1,\ldots, x_l|P). 
$$
If $l = 1$ and $P$ is a uniform distribution in a finite set then 
this definition is equivalently to the one-dimensional case.
The randomness deficiency measures how implausible is to  
get $x_1,\dots,x_l$ as a result of $l$ independent draws from $A$.  
The set off all pairs $(a,b)$ such that there is a distribution $P$ 
of complexity at most $a$ and $d(\overrightarrow{x}|P) \le b$ 
is called the \emph{$l$-dimensional 
stochasticity profile of $\overrightarrow{x}$} and 
is denoted by $Q_{\overrightarrow{x}}$.

It turns out that Theorem \ref{profiles_opt_stoch}
translates to multi-dimensional case:
\begin{theorem}
\label{profiles_opt_stoch_many}
For every tuple $\overrightarrow{x} = x_1, \ldots, x_l $ of strings of length 
$n$ the distance between sets 
$P_{\overrightarrow{x}}$ and $Q_{\overrightarrow{x}}$ is at most $O(\log(n+l))$.
\end{theorem}

The proof of this theorem is presented in Appendix.

\begin{remark}
Theorem~\ref{profiles_opt_stoch_many}
is basically an analog of Theorem \ref{profiles_opt_stoch}
for a restricted class of distributions, 
namely, for product distributions $Q$ on 
$l$-tuples, i.e., distributions of the form 
$Q(x_1,\dots,x_l)=P(x_1)\cdots P(x_l)$. 
A natural question is whether Theorem \ref{profiles_opt_stoch}
can be generalized to any decidable class of distributions.
This is indeed the case and the proof
is very similar to the proof of  Theorem~\ref{profiles_opt_stoch_many}
(presented in Appendix). 
\end{remark}

\section*{An open question}
Can we improve the accuracy in  Corollary \ref{cfl} 
from $2^{O(\log( n + m + l))}$ to  $2^{O(\log (n + l))}$?

\section*{Acknowledgments}

I would like to thank Julia Marakshina, Alexander Shen and 

Nikolay 
Vereshchagin for help in writing this paper.

\section{Appendix}

\begin{proof}[Proof of Theorem \ref{profile_common_inf}]
Our example is borrowed from~\cite{Chernov},
where there are several examples of pairs of strings with non-extractable
common information. All of the examples except one are stochastic
pairs of strings and we need any stochastic such example. 

Consider a finite field $\mathbb{F}$ of cardinality $2^n$ and 
a plane (two-dimensional vector space) over $\mathbb{F}$.
Let $y_1$ be a random line on this  plane, and $y_2$ be a random point on 
this line. Then
$$C(y_1) = 2n, C(y_2) = 2n, C(y_1,y_2) = 3n
$$
(everything with logarithmic accuracy).
These strings $y_1,y_2$ have about $n$ bits of common information.
On the other hand~\cite[Theorem 8]{Chernov} 
states the following:
\begin{theorem}[\cite{Chernov}]  \label{th1}
There is no $z$ such that $C(z)=n+O(\log n)$,  
$C(y_1|z)=n+O(\log n)$, $C(y_2|z)=n+O(\log n)$
(such a string $z$ could be considered as a representation
of the common information in $y_1,y_2$). 
Moreover, for all strings $z$ we have
 \begin{align}\label{eq2}
C(z)+C(y_1|z)/2+\max\{C(y_1|z)/2,C(y_2|z)\}&\ge 3n-O(\log n),\\
C(z)+C(y_2|z)/2+\max\{C(y_2|z)/2,C(y_1|z)\}&\ge 3n-O(\log n).\label{eq3}
\end{align}
\end{theorem}

Let us first show that inequalities~\eqref{eq2} and~\eqref{eq3}
imply that 
 \begin{align}\label{eq4}
C(z)+C(y_1|z)+C(y_2|z)\ge \min\{4n-C(z)/3,5n-C(z)\}-O(\log n).
\end{align}
Indeed, if $C(y_1|z)$ and $C(y_2|z)$ differ 
at most 2 times from each other, then
the maximum in both inequalities~\eqref{eq2} and~\eqref{eq3} is equal to the second
term and summing~\eqref{eq2} and~\eqref{eq3}
we get 
$$
2C(z)+3C(y_1|z)/2+3C(y_2|z)/2\ge 6n-O(\log n),
$$
which can be re-written
as
$$
C(z)+C(y_1|z)+C(y_2|z)\ge 4n-C(z)/3-O(\log n).
$$
Otherwise, when say $C(y_1|z)>2C(y_2|z)$, 
the maximum in inequality~\eqref{eq2} is equal to the 
first term. Then  
we sum that inequality with 
the inequality $C(z)+C(y_2|z)\ge C(y_2)=2n$ and 
obtain the inequality
$$
2C(z)+C(y_1|z)+C(y_2|z)\ge 5n-O(\log n),
$$
which can be re-written
as
$$
C(z)+C(y_1|z)+C(y_2|z)\ge 5n-C(z)-O(\log n).
$$
Thus in both cases we obtain~\eqref{eq4}.

This implies that the optimality profile $P_{y_1,y_2}$ of the pair of strings
existing by Theorem~\ref{th1} has the following property 
\begin{equation}\label{eq6}
\pair{a,b}\in P_{y_1,y_2}\ \Rightarrow\ b\ge \min\{n-a/3,2n-a\}-O(\log n).
\end{equation}
Indeed, for every probability distribution $P$
we have $C(y|P)\le -\log P(y)+O(1)$ and hence
\begin{align}\label{eq5}
\delta(\pair{y_1,y_2},P)\ge C(P)+C(y_1|P)+C(y_2|P)-3n-O(1).
\end{align}
Combining inequality~\eqref{eq4} for  $z=P$ 
and inequality~\eqref{eq5} we obtain~\eqref{eq6}.

Thus the optimality profile of the pair $y_1,y_2$
does not contain the pair $(1.5n,0.5n-O(\log n))$.
On the other hand, all the strings $y_1,y_2,[y_1,y_2]$ are stochastic,
that is, the sets $P_{y_1},P_{y_2},P_{,[y_1,y_2]}$ contain almost all pairs $(a,b)$
(more specifically, all pairs with $a,b\ge O(\log n)$).

It is easy to construct another pair of strings
$x_1,x_2$ that has the same properties except that 
the pair $(n+O(1),O(1))$ is inside $P_{x_1,x_2}$.
To this end let $x_1,x_2$ be random strings 
of length $2n$ that share first $n$ bits: 
$x_1=x^*x^*_1$, $x_2=x^*x^*_2$ and $C(x^*x_1^*x_2^*) = 3n + O(1)$.
Then again $C(x_1)=2n + O(1)$, $C(x_2)=2n + O(1)$, 
$C(x_1x_2) = 3n + O(1)$.
And again  all the strings $x_1,x_2,[x_1,x_2]$ are stochastic.
To show that the pair $(n+O(\log n),O(\log n))$ 
is inside $P_{x_1,x_2}$,
consider the uniform distribution $P$ on all strings of length $n$ whose
first half is equal to $x^*$. 
This distribution has the same complexity as $x^*$, that is,
$C(P)=n+O(1)$ and hence 
$C(P) -\log P(x_1) -\log P(x_2) = 3n+O(1) = C(x_1,x_2)$. Hence
even the pair $(n+O(1), O(1))$ belongs to $P_{x_1,x_2}$.
\end{proof}

\begin{proof}[Proof of Theorem \ref{profiles_opt_stoch_many}]
The proof is similar to the proof of Theorem \ref{profiles_opt_stoch}. 
First notice that for every distribution $P$ 
we have 
$d(\overrightarrow{x}| P) \le \delta(\overrightarrow{x}, P) + O(\log(n+l))$.
Indeed:
\begin{align*}
d(\overrightarrow{x}|P) &= - \log( P(x_1)\dots P(x_n)) - C(\overrightarrow{x}|P) \\
&\le - \log( P(x_1)\dots P(x_n)) + C(P) - C(\overrightarrow{x}) = \delta(\overrightarrow{x}, P).
\end{align*}
Therefore
the set $Q_{\overrightarrow{x}}$ includes the set $P_{\overrightarrow{x}}$ 
(with accuracy $O(\log(n+l))$). 

It remains to show the inverse inclusion.
From the above inequalities it is clear 
that the difference between
$\delta(\overrightarrow{x}, P)$ and $d(\overrightarrow{x}| P)$
equals 
$$
(C(P) - C(\overrightarrow{x}))+ C(\overrightarrow{x}|P)
=C(P|\overrightarrow x),
$$
where the equality follows from the Symmetry of 
information (see, e.g.~\cite{LiVit}).

It turns out that if $C(P|\overrightarrow{x})$ 
is large then there is an explanation $\tilde P$ for
$\overrightarrow{x}$ with much better parameters:
\begin{lemma}
\label{better}
For every distribution $P$ and for every tuple $\overrightarrow{x} = x_1 \ldots x_l$  of strings  of length $n$ there is a distribution $\tilde{P}$ such that:

1) $- \log( \tilde{P}(x_1)\dots \tilde{P}(x_l)) 
\le - \log( P(x_1) \dots P(x_l)) + O(\log(n + l))$ and

2) $C(\tilde{P}) \le C(P) - C(P|\overrightarrow{x}) + O(\log(n+l))$. 
\end{lemma}  

To prove this lemma we need yet another one:
\begin{lemma}
\label{manytosimple}
Let $x_1,\ldots x_l \in \mathbb{B}^n$. Assume, that  there are $2^k$ distributions $P$ such that:

1) $-\log(P(x_1)\dots P(x_l)) \le b$. 

2) $C(P) \le a$.

Then there is a distribution $\tilde{P}$ of complexity at most $a - k + O(\log( n + l + a + b))$ such that $-\log(\tilde P(x_1)\dots \tilde P(x_l)) \le b$.
\end{lemma}
\begin{proof}[Proof of Lemma \ref{manytosimple}]
In Lemma \ref{ml} let $L$ to be  
the set of probability 
distributions and $R$ to be the set 
of $l$-tuples of $n$-bit strings. Then let  
$\pair{x_1, \ldots, x_l}$ be adjacent to $Q$ if 
$\log(Q(x_1)\dots Q(x_l)) \ge -b$.  
\end{proof}
\begin{proof}[Proof of Lemma \ref{better}]
Assume that a tuple $\overrightarrow{x}$
is given. Enumerate all distributions $Q$ such that $C(Q) \le a=C(P)$ and 
$-\log( Q(x_1) \dots Q(x_l) \le b=-\log(P(x_1) \dots P(x_l))$. 
We can retrieve $P$ from $\overrightarrow{x}$ and the ordinal number
of $P$ in this enumerating. Thus the logarithm of that 
number must be greater than 
$C(P|\overrightarrow{x})$ (with logarithmic accuracy). 
By Lemma \ref{manytosimple} for $k = C(P|\overrightarrow{x})$
there is a probability distribution $\tilde P$ in the enumeration
whose complexity is at most $a-k$ (with logarithmic accuracy).
\end{proof}

Now, we are ready to finish the theorem. Consider some distribution $P$. We need to show that there is a distribution $\tilde P$ such that: 
$C(\tilde P) \le C(P) + O(\log(n+l))$ and 
$\delta(\overrightarrow{x},\tilde P) \le d(\overrightarrow{x}|P) + O(\log(n+l))$.
To this end consider the distribution $\tilde{P}$ from Lemma \ref{better}. 
By construction the complexity of $\tilde P$ is at most that of
$P$ (with logarithmic accuracy). And its optimality deficiency
can be bounded as follows:
\begin{align*}
\delta(\overrightarrow{x},\tilde P)&=
C(\tilde P)-\log(\tilde P(x_1)\dots\tilde P(x_l))-C(x_1,\dots,x_l)\\
&\le C(P)-C(P|\overrightarrow{x})-\log(P(x_1)\dots\tilde P(x_l))-C(x_1,\dots,x_l)\\
&=\delta(P,\overrightarrow{x})-C(P|\overrightarrow{x})=d(\overrightarrow{x}|P).
\qed
\end{align*}
\renewcommand{\qed}{}
\end{proof}


\begin{thebibliography}{M}
\bibitem{Chernov}
A. Chernov, An. Muchnik, A. Romashchenko, A. Shen, and N. Vereshchagin. Upper semi-lattice of binary strings with the relation ``x is simple conditional to y''. Theoretical Computer Science 271 (2002) 69--95. 

\bibitem{LiVit}
Li M., Vit\'anyi P., \emph{An Introduction to Kolmogorov
complexity and its applications}, 3rd ed., Springer,
2008 (1 ed., 1993; 2 ed., 1997), xxiii+790~pp.
ISBN 978-0-387-49820-1.

\bibitem{SUV}
A. Shen, V. Uspensky, N. Vereshchagin
{\em Kolmogorov complexity and algorithmic randomness}.
MCCME, 2013 (Russian). English translation:
http://www.lirmm.fr/\~{}ashen/kolmbook-eng.pdf

\bibitem{Shen}
A. Shen
{\em Around Kolmogorov complexity:
basic notions and results}
http://arxiv.org/abs/1504.04955

\bibitem{VerShen}
N. Vereshchagin, A. Shen
{\em Algorithmic statistics revisited}
http://arxiv.org/abs/1504.04950


\bibitem{VerVit}
N.K. Vereshchagin, P.M.B. Vit\'anyi
{\em Rate Distortion a
nd Denoising of Individual
Data Using Kolmogorov Complexity}
IEEE Transactions on Information Theory,56:7 (2010), 3438–3454

\bibitem{VerVit2002}
N. Vereshchagin and P. Vitanyi.
``Kolmogorov's Structure Functions with an Application to the Foundations of Model Selection''. 
IEEE Transactions on Information Theory 50:12 (2004) 3265-3290. 
Preliminary version: Proc. 47th IEEE Symp. Found. Comput. Sci., 2002, 751--760. 

\bibitem{Koppel}
Moshe Koppel
{\em Complexity, depth and sophistication}
Complex Systems 1, pp. 87-91

\bibitem{Kolmogorov}
A.N.~Kolmogorov, Talk
at the Information Theory Symposium in Tallinn, Estonia, 1974.

\bibitem{Sh}
A. Shen
{\em The concept of $(\alpha, \beta)$-stochasticity in the Kolmogorov sense, and its properties.
Soviet Mathematics Doklady, 271(1):295--299, 1983
}

\end{thebibliography}
\end{document}